\documentclass[letterpaper]{article}
 \usepackage{aaai}
 \usepackage{times}
 \usepackage{helvet}
 \usepackage{courier}
\usepackage{amsthm,amsmath,amssymb}
\usepackage{tabularx}
\usepackage{aaai}
\usepackage{algpseudocode}
\usepackage{algorithmicx}
\usepackage{algorithm}

\newcommand{\database}{\mathcal{D}}
\newcommand{\tbox}{\mathcal{T}}
\newcommand{\mappings}{\mathcal{M}}
\newcommand{\schema}{\mathcal{S}}

\newcommand{\eval}{\textup{eval}}

\newcommand{\naf}{\ensuremath\textup{\bfseries not}\,}
\newcommand{\dataquery}{Q^\mathcal{\schema}}
\newcommand{\ontquery}{H^\mathcal{\tbox}}
\newcommand{\negjustification}{J^-}
\newcommand{\posjustification}{J^+}
\newcommand{\varx}{\mathbf{x}}

\newcommand{\vary}{\mathbf{y}}
\newcommand{\varz}{\mathbf{z}}
\newcommand{\vart}{\mathbf t}

\newcommand{\abox}{\mathcal{A}}
\newcommand{\tmodel}{\mathcal{I}} 
\newcommand{\obda}{(\database,\mappings,\tbox)}

\theoremstyle{remark}
\newtheorem{example}{Example}
\newtheorem*{remark}{Remark}
\theoremstyle{plain}
\newtheorem{lemma}{Lemma}
\newtheorem{theorem}{Theorem}
\newtheorem{corollary}{Corollary}

\newenvironment{proofsketch}{\proof}{\endproof}

\theoremstyle{definition}
\newtheorem{definition}{Definition}

\title{Mapping Data to Ontologies with Exceptions Using Answer Set Programming}
\author{Daniel P. Lupp \and Evgenij Thorstensen \\ University of Oslo, Norway\\ \{danielup, evgenit\}@ifi.uio.no}
\begin{document}
 \maketitle 
\begin{abstract}
In ontology-based data access, databases are connected to an ontology
via mappings from queries over the database to queries over the
ontology. In this paper, we consider mappings from relational
databases to first-order ontologies, and define an ASP-based
framework for GLAV mappings with queries over the
ontology in the mapping rule bodies. We show that this type of mappings can be used to express constraints
and exceptions, as well as being a powerful mechanism for succinctly representing OBDA mappings.
We give an algorithm for brave reasoning in this setting, and show that this problem has either the same data complexity as ASP (NP-complete), or it is at least as hard as the complexity of checking entailment for the ontology queries. Furthermore, we show that for ontologies with UCQ-rewritable queries there exists a natural reduction from mapping programs to $\exists$-ASP, an extension of ASP with existential variables that itself admits a natural reduction to ASP.
\end{abstract}

\section{Introduction}

Ontology-based data access (OBDA) \cite{poggi2008} is a method for data integration, utilizing a semantic layer consisting of an ontology and a set of mappings on top of a database.  An \emph{ontology} is a machine-readable
model designed to faithfully represent knowledge of a domain independently of the structure of the database; it is comprised
of concepts and relationships between these concepts. These ontologies are often formulated using \emph{description logics} (DLs), a class of decidable logics, due to their desirable computational properties \cite{DBLP:journals/jar/CalvaneseGLLR07}.

With the help of mappings, users' queries over the ontology are rewritten into a query over the database language, such as SQL, which can then be run on the source data. This query rewriting consists of two stages: Firstly, the ontology query is rewritten to an equivalent query which takes ontological knowledge into account. Secondly, the mappings are used to translate this query into the source query language. To ensure that this rewriting is always possible, one requires the ontology to be \emph{first-order rewritable} (FOL-rewritable); that is, that every rewritten query is equivalent to a first-order formula. However, not all description logics have this property. A common class of ontology languages used in OBDA is the  DL-lite family. These description logics have been tailored towards FOL-rewritability and tractable query answering, making them ideally suited for OBDA \cite{DBLP:journals/jar/CalvaneseGLLR07}.

Unfortuantely, the rewriting step can cause a worst-case exponential blow-up in query size \cite{DBLP:journals/jar/CalvaneseGLLR07}. While this blow-up is necessary to ensure complete query answering, it can lead to highly redundant database queries, where the same data is accessed multiple times. Furthermore, mapping design and maintenance is usually manual work \cite{cbb1fbd58bef414bb51064327f881fc2}. This can be a very laborious task, and recent work on mapping evolution and repair \cite{DBLP:conf/dlog/LemboRSST16} attempt to alleviate some of the difficulties involved. However, currently OBDA mappings are interpreted as first-order implications. As a consequence, they lack the expressivity to efficiently handle these issues: exceptions must be stated explicitly, possibly in multiple mapping assertions. Furthermore, pruning redundant queries without nonmonotonic features such as extensional constraints \cite{raey} or closed predicates \cite{Lutz:2013:ODA:2540128.2540276} is practically infeasible. 

Current research on extending OBDA with nonmonontonic capabilities has focused on the ontology side, e.g., through modal description logics or by inclusion of closed predicates \cite{Donini:2002:DLM:505372.505373,Lutz:2013:ODA:2540128.2540276}. However the modal semantics can be quite unintuitive. In this setting, modal ontology axioms do not behave well with nonmodal axioms. Furthermore, research into extending ontologies with closed predicates quickly results in intractability \cite{Lutz:2013:ODA:2540128.2540276} .

There have also been several approaches to combining rule-based formalisms and description logic ontologies, be it by constructing a hybrid framework integrating both rules and ontology axioms into the same semantics \cite{Motik:2010:RDL:1754399.1754403} or by adding rules ``on top'' \cite{Eiter20081495} of ontologies. Here, the two formalisms retain different semantics, allowing, however, for interaction between rules and ontologies by including special predicates in the rule bodies. 

In this paper, we propose a new framework for OBDA mappings, called \emph{mapping programs}, based on stable model semantics, where mappings are not interpreted as first-order implications. Instead, mappings are rules containing (positive and negative) ontology queries in their bodies, allowing for existential quantification in the body and head of a rule. Each mapping rule contains a database query acting as a guard on the rule, hence existential witnesses generated by mapping rules are not further propagated by the mapping program. This is in contrast to the more general existential rules frameworks of tuple-generating dependencies \cite{DBLP:journals/ws/CaliGL12,DBLP:journals/jair/CaliGK13}, where existentials in heads of rules may propagate. The decidability of mapping program reasoning therefore reduces entirely to decidability of ontology reasoning.

This formalism allows expressing epistemic constraints on the database, such as extensional constraints \cite{raey}. Furthermore, by being able to express default rules, mapping programs serve as a powerful abbreviation tool for mapping maintenance. This allows for adding nonmonotonic features to OBDA while retaining the desirable complexity of ontology reasoning. Mapping programs are a natural extension of $\exists$-ASP \cite{DBLP:conf/ijcai/GarreauGLS15}, an extension of standard answer set programming (ASP) with existential quantifiers in the heads and negative bodies of rules.

In the following, we define and analyze the general mapping program framework, discussing reasoning complexity (NP$^\mathcal{O}$-complete, where $\mathcal{O}$ is an ontology reasoning oracle).  We also consider a special case where the body ontology queries are UCQ-rewritable with respect to the ontology. In this setting, mapping programs can be equivalently reduced to classical ASP. Thus, efficient ASP solvers can be employed for query answering over mapping programs.

\section{Preliminaries}
\subsection{OBDA Mappings}
Let $\Sigma_\tbox$ and $\Sigma_\schema$ be disjoint signatures containing \emph{ontology predicate symbols}, and \emph{source predicate symbols} respectively. Furthermore, let $\mathcal{C}$ be a set of constants. Then a \emph{source schema} $\schema$ is a relational schema containing relational predicates in $\Sigma_\schema$ as well as integrity constraints. A \emph{legal database instance} $\database$ over $\schema$ is a set of ground atoms from $\Sigma_\schema$ and $\mathcal{C}$ that satisfies all integrity constraints in $\schema$. A first-order formula with free variables is called a \emph{query}, if it has no free variables it is called a \emph{boolean query}. An \emph{ontology} $\tbox$ is a set of first-order formulas over $\Sigma_\tbox$. In practice, description logics are often used to express ontologies. Thus, though the results in this paper focus on the general case of FOL ontologies, we will use common DL notation throughout the examples in this paper for notational convenience \cite{DBLP:journals/jar/CalvaneseGLLR07}.

\begin{example}
 The ontology axiom $Boss \sqsubseteq \exists hasSup^-$ is equivalent to the first-order formula $\forall x (Boss(x)\rightarrow \exists y. hasSup(y,x))$. Here, $hasSup^-$ refers to the \emph{inverse role} of $hasSup$.
\end{example}

Following \cite{mappings2}, an OBDA specification is a tuple $\obda$ consisting of a database instance $\database$ legal over a schema $\schema$, a FOL-rewritable ontology $\tbox$ and a set $\mappings$ consisting of \emph{mapping assertions} of the form $m:\varphi \leadsto \psi$, where $\varphi$ and $\psi$ are queries over the data source and ontology, respectively. Then a model $\mathcal{I}$ of an OBDA specification $\obda$ is a first-order model over $\Sigma_\tbox\cup\Sigma_\schema \cup \mathcal{C}$ that satisfies both $\tbox$ and $\mappings$. Here we say that a first-order model $\tmodel$ satisfies a mapping $\mappings$ if $\tmodel \vDash \psi(\vart)$ for every mapping assertion $m:\varphi \leadsto \psi$ and every tuple $\vart\in\textup{eval}(\varphi,\database)$.

\begin{example}\label{ex-shortcomings}
  Consider a database consisting of precisely one two-column table $\texttt{JOBS\_DB(<NAME>,<JOB>)}$. Furthermore, consider the following ontology:
\begin{align*}
  Empl &\sqsubseteq Person\\
  Boss &\sqsubseteq Person
\end{align*}
Suppose that we simply wish to query for all instances of $Person$ in the database.
In the rewriting process, the query $Person(x)$ would be rewritten to 
\begin{align*}
Person(x)&\sqcup Empl(x)\sqcup Boss(x)
\end{align*}
while in the unfolding step, each of the above disjuncts would be expanded to a database query using the mapping assertions. For example, if there exist two mapping assertions $\texttt{JOBS\_DB}(x, ``Accountant") \leadsto Empl(x)$ and $\texttt{JOBS\_DB}(x,``IT")\leadsto Empl(x)$, then the disjunct $Empl(x)$ would be unfolded as $\texttt{JOBS\_DB}(x,``IT")\lor\texttt{JOBS\_DB}(x, ``Accountant")$.
\end{example}
Example~\ref{ex-shortcomings} demonstrates some of the current shortcomings of OBDA: due to its inherent, first-order nature, it is impossible to distinguish between inferred knowledge and knowledge that is explicit in the database. In the above example,  in the presence of a mapping assertion $\texttt{JOBS\_DB}(x,y) \leadsto Person(x)$ the query $Person(x)$ would have sufficed without any ontology rewriting, since all desired information was contained in one table. However, while some OBDA implementations \cite{iswc2015tuning} support manual query pruning, i.e., the user is able to decide which concepts should not be rewritten, this can potentially lead to incomplete query answering, and there is currently no way of formally checking whether it does. Thus, to ensure complete query answering we have a (potentially redundant) worst-case exponential blow-up in query size.
 
Another issue with the current aproach is how exceptions and a lack of information are dealt with. Currently, one must keep track of exceptions manually by explicitly listing all exceptions to a rule. Furthermore, due to the closed-world assumption (CWA) in the database, a lack of knowledge is interpreted as knowledge itself, e.g., if something is not contained in the $\texttt{JOBS\_DB}$ table, it is not a $Person$.

\subsection{Answer Set Programming}

Answer set programming (ASP) is a declarative programming paradigm based on the stable model semantics first defined in \cite{gelfond-lifschitz88} as a means of handling default negation in a straightforward manner. It has become one of the more popular logic programming paradigms, due to, e.g., computational benefits such as guaranteed termination as compared to resolution in Prolog \cite{Lifschitz08whatis}.

An \emph{ASP-program} $P$ s a set of rules of the form

\[
H \leftarrow B_1,\ldots,B_m,\naf C_1,\ldots,\naf C_n.
\]
with ground atoms $H,B_i,$ and $C_j$. The \emph{head} of a rule $r$ is $Head(r)=H$ and the \emph{body} consists of a two parts, the negative body $body^-(r)=\{C_1,\ldots C_n\}$ and the positive body $body^+(r)=\{B_1,\ldots,B_m\}$.  The Herbrand base $HB_P$ of a program $P$ is the set of all possible ground atoms using predicate symbols, function symbols and constants occuring in $P$. Then for a subset $I\subseteq HB_P$, the Gelfond-Lifschitz reduct $P^I$ of $P$ is the set of rules in $P$ after applying the following changes:
\begin{enumerate}
\item If $C_i\in I$ for some $i$, remove the rule (this corresponds to rules that cannot be applied)
\item In all remaining rules, remove the negative clauses $\naf C_i$ (this corresponds to removing all negative clauses that will evaluate to true)
\end{enumerate}
This reduct is a positive program, i.e., a program without any occurence of negation-as-failure. An interpretation $I\subseteq HB_P$ is called a \emph{stable model} or an \emph{answer set} of $P$ if it is a $\subseteq$-minimal model of $P^I$, i.e., it is $\subseteq$-minimal and satisfies all rules in $P^I$.

Though the above semantics require ground atoms, i.e., are essentially propositional, ASP programs might also contains variables or function symbols. In this general case where function symbols are allowed, reasoning becomes undecidable \cite{Alviano11functionsymbols}. In the function-free case, the first-order ASP programs are usually first grounded to reduce it to the propositional case. The grounded programs can then either be solved directly \cite{Gebser:2012:CAS:2228640.2228952} or, e.g., translated to SAT before being passed on to efficient SAT solvers \cite{Lin2004115,gomes2008satisfiability}.
\subsection{$\exists$-ASP}
$\exists$-ASP is an extension to answer set programming proposed by \cite{DBLP:conf/ijcai/GarreauGLS15} to include existential quantification in both the heads and negative bodies of rules. These existential variables are dealt with by Skolemizing and treating the newly introduced function symbols as constants. Specifically, an $\exists$-rule is a rule of the form
\begin{align*}
H_1,\ldots,H_n \leftarrow &B_1,\ldots,B_m,\\&\naf (C^1_1,\ldots,C^1_{u_1}),\ldots,\naf (C^s_1,\ldots,C^s_{u_s}).
\end{align*}
where all $H_i,B_j,C^k_l$ are atoms. Then, all variables not occuring in the positive body of a rule are interpreted existentially. Thus, the \emph{Skolem program} $sk(P)$ of an $\exists$-program $P$ is defined as the set of rules obtained from $P$ by replacing each existential variable in the head of a rule by a Skolem symbol.

Similar to standard ASP, this program is then grounded; one must, however, be careful with existential variables in the negative bodies, as the complete grounding is not equivalent to the nonground rules.
To address this, \cite{DBLP:conf/ijcai/GarreauGLS15} introduce the concept of \emph{partial grounding}, which grounds all variables except the existentials in the negative body. The \emph{reduct} is then defined analogously to the standard ASP case: Given a set of ground atoms $X\subseteq HB_{sk(P)}$, first remove all rules containing a negative body atom which is entailed by $X$. Finally, remove all remaining negative body atoms.

An $\exists$-answer is then defined in the usual way, now also allowing for Skolem symbols in place of constants: a set $X\subseteq HB_{sk(P)}$ is called an \emph{$\exists$-answer set} iff it is a $\subseteq$ minimal model of the reduct $P^X$.

By adding a set $R$ of auxiliary predicates to the signature of an $\exists$-ASP program $P$, one is able to rewrite $P$ into a classical ASP program $P'$ such that they are equivalent with respect to answer sets. In particular, for an $\exists$-answer set $X$ of $P$ there exists some set $A$ of ground atoms over predicates in $R$ such that $X\cup A$ is a classical answer set of $P'$. Furthermore, from an answer set $Y$ of $P'$ one can construct an $\exists$-answer set of $P$ by removing any ground atoms over predicates in $R$ occuring in $Y$ (for more details on the rewriting, see Proposition~8 and the preceding discussion in \cite{DBLP:conf/ijcai/GarreauGLS15}). Therefore, reasoning in $\exists$-ASP can be reduced to reasoning in classical ASP.


\section{OBDA Mapping Programs}
In this section we introduce the syntax and semantics for a new framework for OBDA mappings called \emph{mapping programs}. These programs consist of rules that, intuitively, map database queries $\dataquery$ to ontology queries $\ontquery$ provided that certain conditions $\posjustification$ and $\negjustification$ are met. Thus, mapping programs extend classical OBDA mappings with default reasoning.
 \subsection{Syntax}

   A \emph{mapping rule} is a rule of the form
 \begin{align*}
 \ontquery(\varx,\varz)\leftarrow &\naf \negjustification_1(\vary_1),\ldots,\naf \negjustification_k(\vary_k),\\
&\posjustification_1(\vary'_1),\ldots,\posjustification_l(\vary'_l),\dataquery(\varx).
 \end{align*}
 where $\vary_i,\vary'_j \subseteq \varx$ for all $i,j$. Here, the \emph{head} $\ontquery(\varx,\varz)$ is a first-order formula over $\Sigma_\tbox$ where $\varz$ denotes possible existential variables. The \emph{body} of a mapping rule consists of $\negjustification_i,\posjustification_j$, respectively called the \emph{negative} and \emph{positive justifications} and the \emph{source query} $\dataquery$. Here, $\negjustification_i$ and $\posjustification_j$ are first-order formulas over the language of $\tbox$, and the \emph{source query} $\dataquery$ is a first-order formula over $\Sigma_\schema$. A set $\mappings$ of mapping rules is called a \emph{mapping program}.
\begin{example}\label{ex-syntax}
Consider a database consisting of one table $\texttt {Jobs\_DB(<NAME>,<JOB>)}$. Let $\Sigma_\tbox=\{Empl, hasSup,depHeadOf\}$ with a unary relation $Empl$ of employees and two binary relations $hasSup$ and $depHeadOf$, describing a supervising relation and a department head relation, respectively. The default rule ``employees, of whom we do not know that they are the head of a department, have a supervisor'' can be expressed through the following mapping:
\begin{align*}
m_1:\exists Z.hasSup(X,Z)\leftarrow &\naf \exists Y. depHeadOf(X,Y), \\&Empl(X),\texttt{Jobs\_DB}(X,P).
\end{align*}
\end{example}

Then a \emph{generalized OBDA specification} is a triple $\obda$, where $\database$ is a database instance legal over a source schema $\schema$, $\mappings$ is a mapping program, and $\tbox$ is an ontology.

\subsection{Semantics}\label{sec-abox}

\begin{definition}[Skolem program, following \cite{DBLP:conf/ijcai/GarreauGLS15}]
  Let $\mappings$ be a mapping program. The \emph{Skolem rule} $sk(m)$ associated to a rule $m\in\mappings$ is obtained by replacing each existential variable $v$ in $Head(m)$ by a new Skolem function symbol $sk_v(s)$, where $s$ is an ordered sequence of universal variables in $Head(m)$ . Then the \emph{Skolem program} of $\mappings$ is $sk(\mappings)=\{ sk(m)\mid m\in \mappings\}$.
\end{definition}
 A \emph{mapping interpretation} $\abox$ is a consistent subset of $HB_{sk(\mappings)}$, the Herbrand base over the Skolem program $sk(\mappings)$. Such an interpretation is said to \emph{satisfy} or \emph{model} a positive Skolemized mapping rule 
 \begin{align*}
m: \ontquery(\varx,sk_\varz(\varx))\leftarrow\posjustification_1(\vary'_1),\ldots,\posjustification_l(\vary'_l),\dataquery(\varx).
 \end{align*}
 
written $\abox\vDash m$, if it satisfies the head or does not satisfy the body. It satisfies the body of a rule $m$ if the following holds: for every tuple $\vart\in\eval(\dataquery, \database)$, every interpretation $I$ with $I\vDash \tbox\cup \abox$ satisfies $\posjustification_j[\vart]$ for all $j\leq l$. Here, $\eval(\dataquery, \database)$ denotes the set of tuples $\vart$ that are answers to the query $\dataquery$ over $\database$.
\begin{remark}
  In this framework, the database query $\dataquery$ acts as a guard on the mapping rule $m$. It is in general a first-order query. Since $\dataquery$ is interpreted solely over $\database$, mapping rules are not applicable to existential witnesses generated by mapping rule heads. In particular, the database query $\top(\varx)$ is a shorthand for every tuple $\varx$ occuring in the database.
\end{remark}
For notational brevity, we slightly abuse notation in the following, writing $\mappings$ instead of $sk(\mappings)$. Indeed, in the following we shall only consider the Skolemized mapping program. 
 
An interpretation $\abox$ is said to \emph{satisfy} or \emph{model} a positive mapping program $\mappings$, written $\abox\vDash \mappings$, if it satisfies all mapping rules contained in $\mappings$.

\begin{example}\label{ex-skolem}
  Consider the mapping from Example~\ref{ex-syntax}. By Skolemizing, we get the mapping program:
\begin{align*}
hasSup(X,sk_z(X))\leftarrow &\naf \exists Y. depHeadOf(X,Y),\\ &Empl(X), \texttt{Jobs\_DB}(X,P).
\end{align*}
\end{example}

\begin{definition}[Partial ground program, following \cite{DBLP:conf/ijcai/GarreauGLS15}] The \emph{partial grounding} $PG(m)$ of a mapping rule $m$ is the set of all partial ground instances of $m$ over constants in $\Sigma_\database$ for those variables that are not existential variables in the negative justifications. The \emph{partial ground program} of a mapping program $\mappings$ is the set $PG(\mappings)=\bigcup_{m\in\mappings}PG(m)$.
\end{definition}
\begin{example}\label{ex-grounding}
  Consider the database and mapping from Examples~\ref{ex-syntax} and \ref{ex-skolem}. If the set of constants occuring in the database is $\{a,b\}$, then $PG(sk(m_1))$ consists of the four mapping rules
\begin{align*}
 hasSup(u,sk_z(u))\leftarrow &\naf \exists Y. depHeadOf(u,Y),\\&Empl(u), \texttt{Jobs\_DB}(u,v).
\end{align*}
for $u,v\in \{a,b\}$.
\end{example}
\begin{definition}[$\tbox$-reduct]\label{def-reduct}
Given an ontology $\tbox$, define the $\tbox$-reduct $PG(\mappings)^\abox$ of a partial ground mapping program $PG(\mappings)$ with respect to an interpretation $\abox$ as the mapping program obtained from $PG(\mappings)$ after applying the following:
\begin{enumerate}
\item Remove all mapping rules $m$ where there exists some $i\leq k$ such that $\tbox\cup\abox\vDash  \negjustification_i$.
\item Remove all negative justifications from the remaining rules.
\end{enumerate}
\end{definition}
\begin{example}\label{ex-reduct}
Continuing with our running example, let $\tbox=\{Boss\sqsubseteq \exists depHeadOf , Boss\sqsubseteq \exists hasSup^-
\}$. Furthermore, add the mapping rules
 \begin{align*}
 m_2:Boss(X)&\leftarrow \texttt{Jobs\_DB}(X,b). \\ m_3:Empl(X)&\leftarrow \texttt{Jobs\_DB}(X,P).
 \end{align*}
Then for 
  $\abox =\{\texttt{Jobs\_DB}(a,b), Empl(a), Boss(a)\}$, the rules 
\begin{align*}
  hasSup(a,sk_z(a))\leftarrow &\naf \exists Y. depHeadOf(a,Y),\\&Empl(a), \texttt{Jobs\_DB}(a,v).
\end{align*}
for $v\in \{a,b\}$ are removed in the $\tbox$-reduct $PG(\mappings)^\abox$ construction, since $\tbox\cup\abox \vDash \exists Y. depHeadOf(a,Y)$. Then the $\tbox$-reduct w.r.t. $\abox$ consists of all groundings of the following rules:
\begin{align*}
 hasSup(b,sk_z(b))&\leftarrow Empl(b),\texttt{Jobs\_DB} (b,Y).\\
Boss(X)&\leftarrow \texttt{Jobs\_DB}(X,b).\\
Empl(X)&\leftarrow \texttt{Jobs\_DB}(X,P).
\end{align*}
\end{example}

A mapping interpretation $\abox$ is called a \emph{$\tbox$-answer set} of $\mappings$ if it is a $\subseteq$-minimal model of the $\tbox$-reduct $PG(\mappings)^\abox$.

Then a tuple $(\tmodel,\abox)$ consisting of a first-order model $\tmodel$ and a mapping interpretation $\abox$ is a \emph{model of an generalized OBDA specification $\obda$} if
\begin{enumerate}
\item $\tmodel\vDash \tbox\cup\abox$,
\item $\abox$ is a $\tbox$-answer set of $\mappings$.\label{point2}
\end{enumerate}

For a given ontology $\tbox$, a mapping program $\mappings$ is said to \emph{entail a formula} $\varphi$, written $\mappings \vDash_\tbox \varphi$, if every $\tbox$-answer set of $\mappings$ entails $\varphi$. %

Similarly, a generalized OBDA specification $\obda$ entails a formula $\varphi$, written $\obda \vDash \varphi$, if every model of $\obda$ entails $\varphi$.
\begin{example}
It is easily verifiable that the set $\abox$ given in Example~\ref{ex-reduct} is in fact a $\tbox$-answer set. It does not, however, entail $\tbox$, as the ontology axiom $Boss\sqsubseteq \exists hasSup^-$ is not satisfied. Thus, to obtain a model of the generalized OBDA specification, any model $\tmodel$ must satisfy this axiom, in addition to the assertions in $\abox$.
\end{example}
\begin{remark}[Extensional constraints]
 Mapping programs are capable of expressing extensional constraints over the OBDA specification, i.e., constraints over the ontology language on the database and mappings. For instance, the extensional constraint $C \sqsubseteq_e D$, which in classical OBDA can be intuitively read as ``if $C(a)$ is contained in the ABox, then $D(a)$ is contained in the ABox as well.'' Such a constraints is expressible with the mapping $\bot \leftarrow \naf D(X), C(X), \top(X)$, where $\bot$ is \emph{bottom} and $\top$ is the query \emph{top} of appropriate arity. This guarantees that any $\exists$-answer set of $\mappings$ must satisfy this constraint. It is worth noting that, while this is similar to integrity constraints over the database, it is not entirely the same: the database schema might differ greatly from the structure of the ontology, thus allowing the possibility of describing database constraints on an ontology level.
\end{remark}
\subsection{Complexity Results}
In the general case, where the heads and bodies of mapping rules are allowed to contain arbitrary first-order formulas, reasoning over mapping programs is obviously undecidable. Indeed, consider an empty $\tbox$ and the mapping program $\mappings =\{R(a)\leftarrow \top,H(x)\leftarrow \varphi, R(x)\}$ for some arbitrary first-order formula $\varphi$. Then $\mappings \vDash H(a)$ if and only if $\varphi$ is a tautology, which is known to be undecidable for arbitrary first-order $\varphi$. This is summarized in the following theorem.
\begin{theorem}
The problem of checking $\mappings \vDash A$ for a given mapping program $\mappings$ and a ground atom $A$ is undecidable.
\end{theorem}

\begin{corollary}
  Let $\obda$ be a generalized OBDA specification and $A$ be a ground atom. Then $\obda\vDash A$ is undecidable.
\end{corollary}

Now let $(\tbox, \mathcal{L})$ be a pair consisting of an ontology $\tbox$ and a set $\mathcal{L}$ of formulas over the signature $\Sigma_\tbox$ such that $\tbox$-entailment of any $\varphi\in \mathcal{L}$ is decided by an oracle $\mathcal{O}_{(\tbox,\mathcal{L})}$. In the following we consider mapping programs $\mappings$ where the heads and justifications in rules contain formulas from $\mathcal{L}$. Then to construct a $\tbox$-answer set, we can employ a simple guess-and-check algorithm using the verifier given in Algorithm~\ref{alg-general}.

\begin{algorithm}[h]
\caption{$\tbox$-answer set verifier} 
\label{alg-general}
\algblock[name]{Start}{End}
 \begin{algorithmic}
   \State {\bfseries input} ontology $\tbox$, partially ground Skolem program $\mappings$, set $\abox\subseteq HB_{sk(\mappings)}$:
   \Start
   \State $\mappings^\abox:=$ \Call{make-reduct}{$\abox,\mappings$};
   \If {\Call{check-sat}{$\abox,\mappings^\abox$} and \Call{check-min}{$\abox,\mappings^\abox$}}
   \State \Return true;
   \EndIf
   \State \Return false;
   \End

   \State 
   
  \Procedure{make-reduct}{$\abox,\mappings$}
  \State $\mappings^\abox:= \mappings$;
  \ForAll {$m\in \mappings$}
  \If{$\tbox\cup\abox\vDash\negjustification_i$ for some $i$}
  \State  $\mappings^\abox:=\mappings^\abox\setminus \{m\}$;
  \EndIf
  \EndFor
  \State remove negative clauses from $\mappings^\abox$;
  \State \Return $\mappings^\abox$;
  \EndProcedure

  \State
  
  \Procedure{check-sat}{$\abox,\mappings^\abox$}
  \ForAll {$m'\in\mappings^\abox$}
  \If {$\tbox,\abox\vDash Body(m')$ and $\tbox,\abox\not\vDash Head(m')$}
  \State \Return false; 
  \EndIf
  \EndFor
  \State\Return true;
  \EndProcedure

  \State
  
  \Procedure{check-min}{$\abox,\mappings^\abox$}
  \ForAll {$a\in \abox$}
  \If {\Call{check-sat}{$A\setminus\{a\},\mappings^\abox$}}
  \State \Return false;
  \EndIf
  \EndFor
  \State \Return true;
  \EndProcedure
\end{algorithmic}
\end{algorithm}

Correctness of Algorithm~\ref{alg-general} is obvious: by definition, a set $\abox$ is a $\tbox$-answer set of $\mappings$ if and only if it is a $\subseteq$-minimal model of the $\tbox$-reduct $\mappings^\abox$. Both the construction of $\mappings^\abox$ and the satisfiability-checking are done following the respective definitions. For $\subseteq$-minimality, it is sufficient to check co-satisfiability of $\abox\setminus \{a\}$ for every $a\in \abox$, since $\mappings^\abox$ is a positive program and hence monotonic.

The complexity of Algorithm~\ref{alg-general} depends the complexity of the oracle $\mathcal{O}_{(\tbox,\mathcal{L})}$ and the following factors. Let 

\begin{enumerate}
\item $n^+(\mathcal{N})$ (resp. $n^-(\mathcal{N})$) denote the number of positive (resp. negative) justifications in a mapping program $\mathcal{N}$, and
\item $h(\mathcal{N})$ denote the number of heads in a mapping program $\mathcal{N}$.
\end{enumerate}

Furthermore, let $|\mathcal{O}_{(\tbox, \mathcal{L})}|$ denote the complexity of the oracle $\mathcal{O}_{(\tbox,\mathcal{L})}$. Then each of the three procedures ({\sc make-reduct, check-sat, check-min}) in Algorithm~\ref{alg-general} have the following complexity:

\begin{enumerate}
\item {\sc make-reduct}: the oracle $\mathcal{O}_{(\tbox,\mathcal{L})}$ is called on each negative justification in $\mappings$, so complexity of this procedure is $n^-(\mappings)\cdot |\mathcal{O}_{(\tbox,\mathcal{L})}|$.
\item {\sc check-sat}: the procedure evaluates to true if, for all mapping rules $m$, $\tbox,\abox\vDash Head(m)$ or $\tbox,\abox\not\vDash Body(m)$. Thus, the oracle co-$\mathcal{O}_{(\tbox,\mathcal{L})}$ must be called on all positive justifications in $\mappings^\abox$. If all justifications in a rule are entailed (i.e., the co-oracle evaluates to false), the entailment of the rule heads must be checked. Hence, the complexity is bounded by $n^+(\mappings^\abox)\cdot|\textup{co-}\mathcal{O}_{(\tbox,\mathcal{L})}|+h(\mappings^\abox)\cdot |\mathcal{O}_{(\tbox,\mathcal{L})}|$. 
\item {\sc check-min}: for each $a\in\abox$, co-satisfiability of $A\setminus \{a\}$ must be checked. Thus, there are $|\abox|$ calls to {\sc check-sat}, where the returned value is inverted.
\end{enumerate}
The total complexity of Algorithm~\ref{alg-general} is therefore the sum of the complexities of these three procedures: 

\begin{align*}
&(n^-(\mappings)+h(\mappings^\abox)+|\abox|\cdot n^+(\mappings^\abox))\cdot|\mathcal{O}_{(\tbox,\mathcal{L})}|\\+&(n^+(\mappings^\abox)+|\abox|\cdot h(\mappings^\abox))\cdot |\textup{co-}\mathcal{O}_{(\tbox,\mathcal{L})}|.
\end{align*}

Algortihm~\ref{alg-general} is in fact a generalization of the verifier used in the guess-and-check method for classical ASP: Indeed, consider the case where $\tbox=\emptyset$ and $\mathcal{L}$ is the set of all ground atoms over the language of $\tbox$. In this case, the oracle $\mathcal{O}_{(\tbox,\mathcal{L})}$ must only check membership in $\abox$, hence it is linear in the size of $\abox$. Thus, in this setting a partially ground Skolem mapping program is simply a classical ASP program. Therefore, brave reasoning over mapping programs is at least as hard as classical ASP solving, i.e., is NP-hard \cite{Lifschitz08whatis}.

More generally, for a given reasoning oracle $\mathcal{O}_{(\tbox,\mathcal{L})}$ brave reasoning over mapping programs is NP$^{\mathcal{O}_{(\tbox,\mathcal{L})}}$-complete. 
\begin{theorem}\label{thm-complex}
  Let $(\tbox,\mathcal{L})$ be a pair consisting of an first-order ontology $\tbox$ and a set of formulas $\mathcal{L}$ over the language of $\tbox$ such that $\tbox$-entailment is $|\mathcal{O}_{(\tbox,\mathcal{L})}|$-hard for an oracle $\mathcal{O}_{(\tbox,\mathcal{L})}$. Then for a partially ground Skolemized mapping program $\mappings$ where the head and all justifications are formulas from $\mathcal{L}$, $\tbox$-answer set existence is NP$^{\mathcal{O}_{(\tbox,\mathcal{L})}}$-complete.
\end{theorem}
\begin{proofsketch}
  Intuitively, a NP$^{\mathcal{O}_{(\tbox,\mathcal{L})}}$ Turing machine (that is, an NP Turing machine that allows for $\mathcal{O}_{(\tbox,\mathcal{L})}$-calls on the tape) can be encoded as a mapping program in the same manner as an NP Turing machine can be encoded in classical ASP, however allowing for oracle calls in the mapping rules' bodies. \end{proofsketch}
It is worth noting that, by the preceding theorem, a partially grounded Skolemized mapping program satisfying the conditions of Theorem~\ref{thm-complex} can be rewritten into an ASP program with oracle calls in the rule bodies. The resulting program, however, bears little resemblance to the original mapping program, as it is the encoding of the NP$^{\mathcal{O}_{(\tbox,\mathcal{L})}}$ Turing machine.
\subsection{UCQ-Rewritable Justifications}

We now analyze a restriction of mapping programs that admit a natural reduction to classical ASP for query answering and reasoning. To this end, let $\tbox$ be an ontology over a decidable fragment of first-order logic. We say a formula $\varphi$ over $\Sigma_\tbox$ is \emph{UCQ-rewritable with respect to $\tbox$} if the $\tbox$-rewriting of $\varphi$ is equivalent to a union of conjunctive queries \cite{DiPinto:2013:OQR:2452376.2452441}.

 Then for a mapping program $\mappings$ where all justifications are UCQ-rewritable with respect to $\tbox$, let $\overline \mappings$, called the \emph{$\tbox$-rewritten program}, denote the mapping program obtained from $\mappings$ by replacing every justification with its rewriting with respect to $\tbox$. The $\tbox$-rewritten program $\overline \mappings$ is equivalent to a program containing only atoms as positive justifications and CQs as negative justifications, by well-known logic program equivalence transformations \cite{Lifschitz1999}. By abuse of notation, $\overline \mappings$ will in the following denote this equivalent program.

Let us first establish the connection between mapping programs and $\exists$-ASP. Recall that a mapping rule can be applied to every tuple $\vart\in\eval(\dataquery, \database)$ where  $\tbox\cup\abox\vDash \posjustification_i[\vart]$ for all positive justifications $\posjustification_j$ and  $\tbox\cup\abox\not\vDash \negjustification_j[\vart]$ for all negative justifications $\negjustification_j$. If the TBox $\tbox$ is empty, this statement reduces to checking whether the justifications are certain answers w.r.t. $\abox$ and hence simply checking containment in $\abox$. This is, however, precisely the semantics of ASP with existential variables in the heads and negative bodies of rules. Hence, mapping programs can be seen as an extension of $\exists$-ASP, both semantically and syntactically. This result is summarized in the following theorem.
\begin{theorem}\label{thm-toexist}
 Let $\mappings$ be a partially ground Skolem program where all justifications are conjunctive queries. Then a set $\abox$ is a $\emptyset$-answer set of $\mappings$ iff it is a $\exists$-answer set of $\mappings$. 
\end{theorem}
The following lemma describes the relationship between $\tbox$-rewritten programs and reducts w.r.t. $\abox$, which is particularly useful when analyzing the connection between $\exists$-ASP and mapping programs, as discussed in Theorem~\ref{thm-tbox-exasp}.
\begin{lemma}\label{lemma-m-overlinem}
  For any $\abox\subseteq HB_{sk(\mappings)}$ the equality $\overline \mappings^\abox=\overline{\mappings^\abox}$ holds, where $\overline{\mappings^\abox}$ denotes the $\tbox$-rewritten program of $\mappings^\abox$.
\end{lemma}
\begin{proof}
  Let $m$ be a mapping rule removed from $\mappings$ in the construction of the $\tbox$-reduct $\mappings^\abox$, i.e., there exists some $i\leq k$ such that $\tbox\cup \abox \vDash \negjustification$. This is equivalent to $\emptyset,\abox\vDash \overline{\negjustification_i}$ where $\overline{\negjustification_i}$ is the $\tbox$-rewriting of $\negjustification_i$. Thus $\overline m$ is removed from $\overline \mappings$ in the construction of the $\tbox$-reduct $\overline \mappings^\abox$. Hence $\overline m\in\overline\mappings^\abox$ iff $m\in\mappings^\abox$ iff $\overline m \in \overline{\mappings^\abox}$.
\end{proof}

\begin{theorem}\label{thm-tbox-exasp}
 Let $\mappings$ be a partially ground Skolem program where all justifications are UCQ-rewritable with respect to an ontology $\tbox$. A set $\abox$ is a $\tbox$-answer set of $\mappings$ iff it is an $\emptyset$-answer set of $\overline\mappings$.
\end{theorem}
\begin{proof}
Let $\abox\subseteq HB_{sk(\mappings)}$ and let \[r:\ontquery[\vart,sk_\varz(\vart)]\leftarrow\posjustification_1[\vart],\ldots,\posjustification_k[\vart]\]
 be any rule in $\mappings^\abox$. We shall prove the statement by separately showing (1) the equivalence of rule satisfaction in $\mappings^\abox$ and $\overline\mappings^\abox$ and (2) that $\abox\vDash \mappings^\abox$ is minimal iff $\abox\vDash \overline\mappings^\abox$ is minimal.
\begin{enumerate}
\item \emph{Satisfaction:}\label{enumsat}  
\begin{align*}
    &\abox \vDash r\\
                             \Leftrightarrow &\textup{ if }\abox\vDash \dataquery[\vart] \textup{ and } \abox,\tbox \vDash \posjustification_i[\vart] \textup{ for all } i\leq k \\
                              &\textup{ then } \abox\vDash \ontquery[\vart,sk_\varz(\vart)].\\
   \Leftrightarrow &\textup{ if }\abox\vDash \dataquery[\vart] \textup{ and } \abox,\emptyset \vDash \overline{\posjustification_i}[\vart] \textup{ for all } i\leq k \\
                              &\textup{ then } \abox\vDash \ontquery[\vart,sk_\varz(\vart)].\\
 \Leftrightarrow &\abox\vDash \overline r
  \end{align*}

\item \emph{Minimality:} Assume $\abox'\subsetneq \abox$ is a model of $\mappings^\abox$. By \ref{enumsat}, this is the case if and only if $\abox'\vDash \overline{\mappings^\abox}$. Finally, Lemma~\ref{lemma-m-overlinem} yields the desired result that $\abox'\vDash \overline\mappings^\abox$.
\end{enumerate}
\end{proof}

As a direct consequence of the preceding theorem, the following corollary describes how query answering over an OBDA specification using a UCQ-rewritable mapping program can be reduced to query answering over an equivalent OBDA specification with an empty ontology.

\begin{corollary}\label{cor-equiv}
Let $\obda$ be an OBDA specification, $\overline\mappings$ the $\tbox$-rewritten program of $\mappings$, $q$ a query over $\Sigma_\tbox$, and $\overline q$ its rewriting with respect to $\tbox$. Then
\[
\obda \vDash q[t]\iff (\database, \overline\mappings, \emptyset) \vDash \overline q[t]
\]
\end{corollary}

Therefore, by Corollary~\ref{cor-equiv} and Theorem~\ref{thm-toexist} we find that every UCQ-rewritable mapping program $\mappings$ is equivalent (w.r.t. answer sets) to an $\exists$-ASP program. By results in \cite{DBLP:conf/ijcai/GarreauGLS15}, this can be further reduced to a classical ASP program. This is summarized in the following theorem.

\begin{theorem}
  For an OBDA specification $\obda$, where the justifications in $\mappings$ are UCQ-rewritable with respect to $\tbox$, there exists an ASP program $\mappings'$ such that for a query $q$ over $\tbox$ 
\[
\obda \vDash q[t] \iff \mappings' \vDash \overline q[t],
\]
i.e., query answering over $\obda$ reduces to cautious reasoning over $\mappings'$.
\end{theorem}
\section{Conclusion and Future Work}
In this paper, we propose a new mapping framework for ontology-based data access (and data transformation in general) that greatly enhances the mappings' expressivity. Our framework allows for default reasoning over the database and ontology, as well as the expression of various epistemic properties of the database, such as extensional constraints and closed predicates. We have shown that in the case where the rule body is UCQ-rewritable, this framework can be rewritten to an equivalent $\exists$-ASP program, and hence query answering reduces to cautious reasoning over ASP.

While various highly optimized ASP solvers do exist, the data complexity involved is rather undesirable in the context of real-world OBDA and big data. Therefore, one of the greatest priorities regarding future work is to determine how and when the complexity can be reduced; the mapping program should not be run on the entire data set. This could, for instance, be addressed by splitting the program into two parts, an easily solvable and a more difficult subprogram, and caching which ontology concepts are easily unfolded.

In addition to such considerations, a prototype should be implemented to demonstrate the feasibility in real contexts, and compare our framework to existing approaches.

\paragraph{Acknowledgments}{We would like to thank the anonymous referees for their very insightful comments and suggestions.}
 \bibliographystyle{aaai}
 \bibliography{bibliography}
  
\end{document}